\documentclass{article} 
\usepackage{iclr2027_conference,times}

\usepackage{amsmath}
\let\castOriginalEqref\eqref

\usepackage{amsmath,amsfonts,bm}

\def\eqref#1{equation~\ref{#1}}

\def\1{\bm{1}}

\DeclareMathAlphabet{\mathsfit}{\encodingdefault}{\sfdefault}{m}{sl}
\SetMathAlphabet{\mathsfit}{bold}{\encodingdefault}{\sfdefault}{bx}{n}

\let\eqref\castOriginalEqref

\usepackage{url}
\usepackage{amsthm}
\usepackage{amsmath}
\usepackage{amssymb}
\usepackage{algorithm}
\usepackage{algorithmic}
\usepackage{graphicx}
\usepackage{adjustbox}
\usepackage{dblfloatfix}
\usepackage{flafter}
\usepackage{float}
\usepackage{placeins}

\definecolor{cvprblue}{rgb}{0.21,0.49,0.74}
\usepackage[pagebackref,breaklinks,colorlinks,allcolors=cvprblue]{hyperref}
\usepackage{enumitem}
\usepackage{booktabs}
\usepackage{multirow}
\usepackage{xcolor}
\usepackage{tabularx}
\usepackage{tcolorbox}
\newtheorem{proposition}{Proposition}

\newcommand{\best}[1]{\textbf{#1}}

\title{CAST: Credit Alignment for Cross-Tokenizer On-Policy Distillation}

\author{\textbf{Jingru Yang}$^{1}$ \quad \textbf{Jinhang Wu}$^{1}$ \quad
\textbf{Minjing Shi}$^{1}$ \quad \textbf{Shengfeng He}$^{2}$\\
\textbf{Xuyan Li}$^{1}$ \quad \textbf{Man Tan}$^{1}$ \quad \textbf{Biye Li}$^{1}$\\[3pt]
$^{1}$Taobao \& Tmall Group of Alibaba\\
$^{2}$Singapore Management University\\[3pt]
\footnotesize\texttt{yjr543596@alibaba-inc.com}, \texttt{jinhang.wjh@taobao.com}\\
\footnotesize\texttt{shiminjing.smj@alibaba-inc.com}, \texttt{shengfenghe@smu.edu.sg}\\
\footnotesize\texttt{lixuyan.lxy@alibaba-inc.com}, \texttt{tanman.tm@alibaba-inc.com}\\
\footnotesize\texttt{libiye.lby@alibaba-inc.com}}

\iclrfinalcopy 

\hypersetup{
  pdftitle={CAST: Credit Alignment for Cross-Tokenizer On-Policy Distillation},
  pdfauthor={Jingru Yang, Jinhang Wu, Minjing Shi, Shengfeng He, Xuyan Li, Man Tan, Biye Li}
}

\begin{document}
\maketitle
\lhead{Preprint}
\begingroup
\vspace{-0.3in}
\setlength{\intextsep}{0pt}
\setlength{\abovecaptionskip}{3pt}
\setlength{\belowcaptionskip}{0pt}
\begin{figure}[H]
    \centering
    \includegraphics[width=0.95\linewidth]{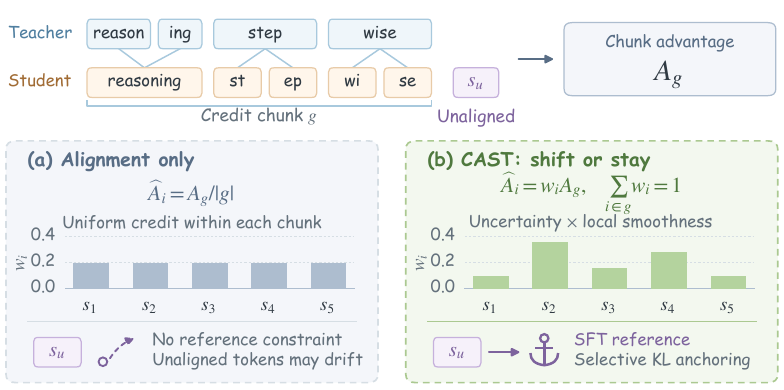}\vspace{-2mm}
    \caption{\textbf{Chunk-level alignment does not resolve credit alignment.}
\textbf{(a)} Alignment-only propagation uniformly distributes each chunk's
advantage across aligned student tokens, leaving unaligned tokens unconstrained.
\textbf{(b)} CAST assigns credit using student uncertainty and local
teacher-signal smoothness, while anchoring unaligned tokens to the supervised fine-tuning (SFT) reference with selective KL regularization.}
    \vspace{-1mm}
    \label{fig:intro_teaser}
\end{figure}
\endgroup

\begin{abstract}
Cross-tokenizer on-policy distillation (OPD) requires aligning teacher and student sequences before their scores can be compared. Decoded-text alignment provides such correspondence at the chunk level, but leaves a distinct problem unresolved: \emph{credit alignment}. A chunk-level signal specifies how much a span should change, but not how that signal should be distributed across student tokens, while tokens without teacher correspondence remain unconstrained. We propose \textbf{CAST} (\textbf{C}hunk-aware credit \textbf{A}ssignment with \textbf{S}elective \textbf{T}oken Anchoring), a \emph{shift-or-stay} framework that addresses both issues. For aligned chunks, CAST redistributes each chunk's advantage across student tokens using predictive uncertainty and local teacher-signal smoothness, emphasizing learning opportunities supported by locally consistent teacher feedback while preserving the chunk's total advantage. The resulting weights exactly solve a KL-regularized credit-allocation problem. For unaligned positions, CAST selectively anchors the student to its initial policy through a KL penalty, limiting drift where teacher supervision is unavailable. Because credit assignment is separated from chunk construction, CAST also applies to punctuation-defined or fixed-length chunks under shared tokenization. Across six mathematics, science, and code benchmarks, CAST improves the average over the strongest evaluated baseline by $2.45$ and $1.32$ points for cross-tokenizer transfer from Qwen3-8B to Mistral-7B and Llama-3.1-8B, respectively, and by $2.10$ points for shared-tokenizer transfer to Qwen3-1.7B.
\end{abstract}

\begin{figure}[t]
\centering
\includegraphics[width=\linewidth]{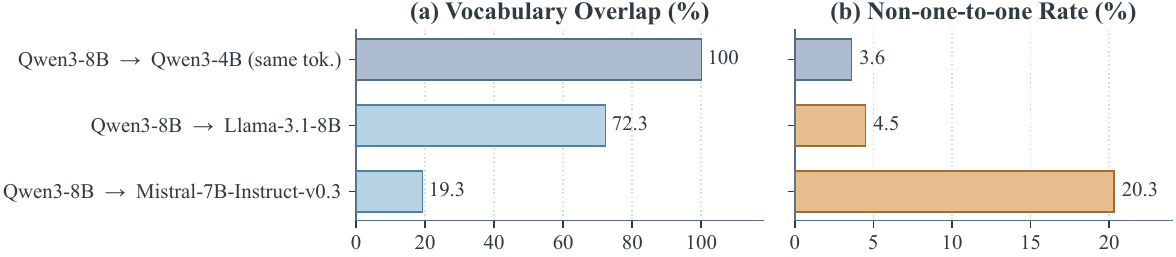}\vspace{-2mm}
\caption{\textbf{Tokenizer mismatch persists despite vocabulary overlap.}
\textbf{(a)} Teacher vocabulary with a decoded-content match in the student.
\textbf{(b)} Student response tokens without one-to-one teacher correspondence
under heterogeneous decoded-text alignment
(Appendix~\ref{app:chunk_construction}).
The shared-tokenizer $3.6\%$ residual results from decode/re-encode
normalization and special/boundary-token handling; training instead uses
identity correspondence.}\vspace{-5mm}
\label{fig:tokstats}
\end{figure}

\section{Introduction}

On-policy distillation (OPD) trains a student on trajectories sampled from its
own policy using feedback from a stronger teacher
\citep{agarwal2024policy, gu2024minillm, pmlr-v235-ko24c, jang2026stable}.
Because supervision is applied to states the student actually visits, OPD can
target student-specific errors and reduce the train--inference mismatch induced
by learning only from fixed teacher-generated responses. It has therefore
become an increasingly useful paradigm for language-model post-training,
particularly for transferring reasoning ability
\citep{yang2025qwen3, bousselham2026vold}.

A basic difficulty arises when teacher and student use different tokenizers.
With a shared tokenizer, the token-wise log-density ratio
$\log \pi_T(y_t\mid x,y_{<t})-\log \pi_{\theta_{\mathrm{old}}}(y_t\mid x,y_{<t})$
compares probabilities of the same next-token event. Across tokenizers, however,
the same text can be segmented differently, so one student token may correspond
to multiple teacher tokens, or vice versa. Teacher and student logits are also
defined over different vocabularies, making direct token-wise comparison
ill-defined. As shown in Figure~\ref{fig:tokstats}, only $19.3\%$ of the
Qwen3-8B vocabulary has a decoded-content match in
Mistral-7B-Instruct-v0.3, compared with $72.3\%$ in Llama-3.1-8B.
Moreover, vocabulary overlap alone does not imply token correspondence:
identical IDs need not represent the same event, and identical text can still
admit different segmentations.

Existing cross-tokenizer distillation methods address this mismatch by
constructing a common space in which teacher and student scores can be compared
\citep{boizard2025towards, cui2025multi, minixhofer2026universal, singh2026cross}.
Dynamic Probabilistic Cross-tokenizer Alignment (DPCA), for example, extends
decoded-text alignment to OPD by matching teacher and student sequences into
text-equivalent chunks and aggregating their token probabilities within each
chunk \citep{niu2026tokenizerbarrier}. Such methods establish
\emph{chunk-level alignment}: the teacher and student scores now refer to the
same decoded text.

Our key observation is that \emph{chunk alignment does not resolve credit
alignment}. Alignment determines \emph{what} teacher and student quantities
are comparable, but OPD ultimately optimizes individual student-token
probabilities. A chunk-level discrepancy therefore still has to be translated
into token-level updates. Uniform propagation gives every student token in a
chunk the same advantage, implicitly treating all positions as equally useful
and equally reliable learning targets. This need not hold. Some positions are
already predicted confidently by the student, whereas others are uncertain and
offer greater opportunity for correction. At the same time, high student
uncertainty alone is insufficient: it can coincide with abrupt local variation
in the teacher-derived signal, where aggressively concentrating credit may be
undesirable. A second issue appears outside aligned regions. Tokens without
teacher correspondence receive no direct supervision, yet their predictions can
still drift because they share parameters with supervised positions. Thus,
decoded-text alignment leaves two coupled optimization questions:
\emph{where should aligned chunk credit be assigned, and how should the student
behave where teacher evidence is absent?}

We address these questions with
\textbf{CAST} (\textbf{C}hunk-aware credit \textbf{A}ssignment with
\textbf{S}elective \textbf{T}oken Anchoring), a token-level
\emph{shift-or-stay} framework for OPD
(Figure~\ref{fig:intro_teaser}).
At aligned positions, CAST selectively \emph{shifts} the student toward the
teacher by redistributing each chunk's advantage across its student tokens.
Its credit rule combines two complementary signals:
student predictive uncertainty identifies \emph{where learning may be useful},
while local teacher-signal smoothness identifies \emph{where that supervision
is locally consistent}. Their product therefore favors uncertain student
predictions supported by a stable local teacher signal. Normalization within
each chunk preserves the chunk's total advantage, so CAST changes only
\emph{where} the credit is assigned, not \emph{how much} aggregate credit the
chunk receives. Moreover, the resulting weights exactly solve a
KL-regularized allocation problem that trades off an uncertainty-based prior
against local teacher-signal variation, with bounded relative reweighting from
the smoothness term.

At unaligned positions, CAST instead encourages the student to \emph{stay}
close to its initial policy. Selective token anchoring applies a KL penalty
toward the supervised fine-tuning (SFT) initialization only where aligned
teacher supervision is unavailable, constraining incidental drift without
directly competing with teacher-guided updates. Together, the two components
separate the roles of teacher supervision and reference preservation:
\emph{shift where teacher evidence is available, stay where it is not}.

This perspective also makes CAST applicable beyond tokenizer mismatch.
Credit assignment is conceptually separate from how chunks are constructed.
Decoded-text alignment provides chunks under heterogeneous tokenization, but
the same allocation rule can operate on punctuation-defined or fixed-length
chunks when teacher and student share a tokenizer. CAST therefore addresses a
more general problem: converting span-level learning signals into structured
token-level updates. We additionally find that multi-rollout sampling can
further improve performance by exposing the credit mechanism to multiple
student trajectories per prompt.

We evaluate CAST on six benchmarks spanning mathematics, science, and code:
AIME~2024, AIME~2025, AIME~2026 \citep{aime26}, MATH-500
\citep{hendrycks2021measuring}, GPQA-Diamond
\citep{rein2024gpqa}, and LiveCodeBench
\citep{jain2025livecodebench}.
For cross-tokenizer distillation from Qwen3-8B to Mistral-7B and
Llama-3.1-8B, CAST improves the six-benchmark average over the strongest
evaluated baseline by $2.45$ and $1.32$ points, respectively. Under shared
tokenization, CAST improves the corresponding average for Qwen3-1.7B by
$2.10$ points. On Mistral, its advantage over uniform and uncertainty-only
credit controls persists across evaluated seeds, together with fewer observed
training collapses
(Table~\ref{tab:seed_summary}; Appendix~\ref{app:seed_robustness}). On Llama,
CAST also outperforms ULD across all three evaluated seeds
(Table~\ref{tab:llama_seed_summary}; Appendix~\ref{app:seed_robustness}).

Our contributions are:
\begin{itemize}[leftmargin=*,topsep=0pt, partopsep=0pt]
    \item We distinguish \emph{chunk-level alignment} from
    \emph{credit alignment} in cross-tokenizer OPD. Chunk alignment establishes
    comparable teacher and student signals, but leaves their token-level
    allocation and the behavior of unaligned positions unresolved.

    \item We propose CAST, a \emph{shift-or-stay} framework that allocates
    chunk credit according to student uncertainty and local teacher-signal
    smoothness while selectively anchoring unaligned positions. The allocation
    preserves each chunk's total advantage and admits an exact
    KL-regularized optimization interpretation.

    \item We show that credit alignment is useful beyond heterogeneous
    tokenization by applying the same rule to punctuation-defined and
    fixed-length chunks under shared tokenization. Across six benchmarks,
    CAST outperforms the strongest evaluated baseline for both heterogeneous
    students and shared-tokenizer Qwen3-1.7B, with additional studies of
    robustness, training dynamics, and component interactions.
\end{itemize}
\section{Related Work}
\label{sec:related}

\noindent\textbf{On-policy distillation.}
Generalized knowledge distillation (GKD) supervises student-generated sequences
with teacher token distributions and serves as our shared-tokenizer baseline
\citep{agarwal2024policy}. MiniLLM adopts a reverse-KL objective, while
DistiLLM and DistiLLM-2 develop streamlined and contrastive objectives
\citep{gu2024minillm, pmlr-v235-ko24c,
3780338.3781556}. KDRL combines distillation with
reinforcement learning for reasoning post-training
\citep{xu2025kdrl}. Training stability has also been
studied through group-relative baselines from multiple responses per prompt
and adaptive target reformulation \citep{guo2025deepseek, jang2026stable}.
These methods assume a shared tokenizer. In this setting, CAST retains its
token-level credit rule and uses punctuation-defined or fixed-length chunks.

\begin{figure}[t]
    \centering
    \includegraphics[width=0.95\textwidth]{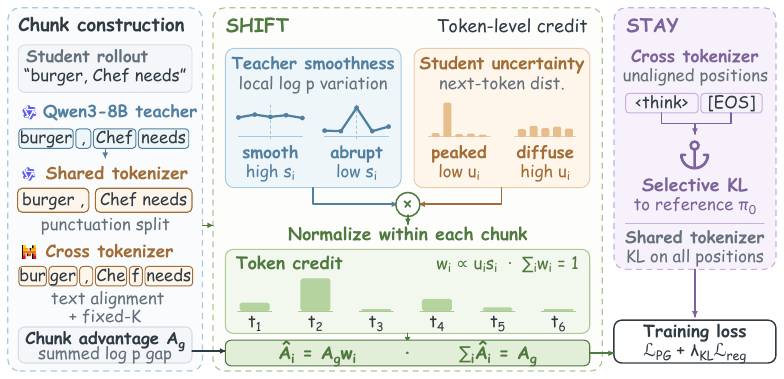}\vspace{-2mm}
    \caption{\textbf{CAST: token-level shift-or-stay optimization.}
\textbf{Shift} normalizes uncertainty $u_i$ times smoothness $s_i$ to allocate
token credit while preserving each chunk's advantage. \textbf{Stay} anchors
unaligned predictions to the initial policy. Heterogeneous pairs use decoded-text
alignment and fixed-length credit chunks; shared-tokenizer training uses
punctuation-defined chunks here and regularizes all positions.}\vspace{-3mm}
    \label{fig:method_overview}
\end{figure}

\noindent\textbf{Cross-tokenizer distillation and credit assignment.}
ULD compares sorted logit distributions through
optimal transport, and MultiLevelOT extends transport across multiple
granularities \citep{boizard2025towards, cui2025multi}. ALM matches likelihoods
over shared text chunks, while byte-level distillation uses a
tokenizer-independent byte interface \citep{minixhofer2026universal, singh2026cross}.
Minimum edit-distance (MinED) token mapping supports knowledge fusion
\citep{wan2024knowledge}. These methods primarily determine
\emph{what} can be compared across tokenizers. DPCA extends decoded-text
alignment to on-policy distillation by matching teacher and student sequences
into text-equivalent chunks and propagating aggregate teacher scores
\citep{niu2026tokenizerbarrier}. Shared-tokenizer methods use local divergence
reweighting or adaptive policy distillation
\citep{jung2025todi, zhang2025aligndistil}.
E-SCOT constrains offline optimal transport by character overlap and reweights
token-level costs using teacher--student R\'enyi-entropy gaps
\citep{liu2026entropyawarespanconstrained}. CAST addresses token-level
\emph{shift-or-stay} optimization after decoded-text alignment. It redistributes
each chunk's on-policy advantage to individual tokens, using uncertainty to
highlight potential learning opportunities and smoothness to temper credit
where teacher scores vary abruptly, while preserving the total advantage.
Selective KL anchoring keeps unaligned predictions close to the initial
student policy, constraining drift from shared parameter updates.

\section{Method}
\label{sec:method}

\subsection{Preliminaries}
\label{sec:formulation}
\label{sec:alignment}

We distill a teacher $p_T$ into a student policy $\pi_\theta$ with vocabulary
$\mathcal V_S$; the two models may use different tokenizers. For prompt $x$, the rollout policy samples
$y=(a_1,\ldots,a_N)\sim\pi_{\theta_{\rm old}}(\cdot\mid x)$ and records
$\ell_i^S=\log\pi_{\theta_{\rm old}}(a_i\mid x,a_{<i})$. The teacher
retokenizes the decoded response as $(b_1,\ldots,b_M)$ and returns
$\ell_j^T=\log p_T(b_j\mid x^T,b_{<j})$ under its own input format $x^T$.

Following DPCA~\citep{niu2026tokenizerbarrier}, we match decoded text across
tokenizers to form \emph{alignment chunks} and obtain uniformly projected teacher
scores $\widetilde\ell_i^T$ on aligned student positions $\mathcal A$, leaving
$\mathcal U$ unaligned. We regroup aligned tokens into $K$-token
\emph{credit chunks} $g\in\mathcal G$
($K=6$ by default; Figure~\ref{fig:method_overview}, left).

\label{sec:homogeneous_method}
With shared tokenization, $M=N$, $\widetilde\ell_i^T=\ell_i^T$, and
$\mathcal U=\varnothing$; $\mathcal G$ uses punctuation-defined or fixed-length
chunks. Construction details are in Appendix~\ref{app:chunk_construction}.

\subsection{Uncertainty--Smoothness Credit Assignment}
\label{sec:credit_assignment}

The \emph{shift} component (Figure~\ref{fig:method_overview}, center)
redistributes each credit chunk's advantage to individual student tokens.
For $g\in\mathcal G$, we define the chunk's
teacher--student log-probability gap and token advantages as
\begin{equation}
    A_g=\sum_{i\in g}(\widetilde{\ell}_i^T-\ell_i^S),
    \qquad
    \widehat A_i=A_gw_i,\quad w_i\geq0,\quad\sum_{i\in g}w_i=1.
    \label{eq:chunk_advantage}
\end{equation}
\textbf{Student uncertainty} uses rollout-policy entropy to highlight
potential learning opportunities:
\begin{equation}
    H_i=-\sum_{v\in\mathcal V_S}\pi_{\theta_{\rm old}}(v\mid x,a_{<i})
    \log\pi_{\theta_{\rm old}}(v\mid x,a_{<i}),
    \qquad u_i=\max(H_i,10^{-8}).
    \label{eq:entropy}
\end{equation}
\textbf{Teacher-signal smoothness} tempers credit where teacher scores vary abruptly.
For ordered $g=(i_1,\ldots,i_{K_g})$ with $K_g=|g|$ and
$q_k=\widetilde\ell_{i_k}^T$, define
\begin{equation}
\delta_k=\begin{cases}
0,&K_g=1,\\
|q_2-q_1|,&k=1,\ K_g>1,\\
|q_{K_g}-q_{K_g-1}|,&k=K_g,\ K_g>1,\\
(|q_k-q_{k-1}|+|q_{k+1}-q_k|)/2,&1<k<K_g.
\end{cases}
\label{eq:local_delta}
\end{equation}
For heterogeneous tokenizers, projected scores are constant within each
alignment chunk (Eq.~\eqref{eq:teacher_projection}, Appendix~\ref{app:chunk_construction}).
Thus $\delta_k$ reflects transitions between alignment chunks within credit
chunk $g$. With a shared tokenizer,
Eq.~\eqref{eq:local_delta} measures adjacent teacher-token score variation.
We combine uncertainty with inverse local variation:
\begin{align}
    \epsilon_g&=\max\left(\frac{1}{K_g}\sum_{k=1}^{K_g}\delta_k,
    0.1\right),\qquad s_{i_k}=\frac{1}{\delta_k+\epsilon_g},
    \label{eq:smoothness}\\
    w_{i_k}^{\mathrm{US}}&=
    \frac{u_{i_k}s_{i_k}}{\sum_{r=1}^{K_g}u_{i_r}s_{i_r}},
    \qquad \widehat A_{i_k}=A_gw_{i_k}^{\mathrm{US}}.
    \label{eq:us_credit}
\end{align}
The adaptive floor limits inverse weights at small variations.
Appendix~\ref{app:token_credit} illustrates the resulting token allocation
on a recorded rollout (Figure~\ref{fig:credit_case}).
Smoothness is a structural prior, not an estimate of teacher correctness.
All weighting quantities are stop-gradient.
Constant $u_i$ or $\delta_i$ recovers smoothness-only or uncertainty-only
allocation, respectively; if both are constant, $w_i=1/|g|$.

\paragraph{Variational characterization and bounded distortion.}
For credit chunk $g$, define the
uncertainty prior $p_i^u=u_i/\sum_{j\in g}u_j$, the dimensionless local
irregularity $\rho_i=1+\delta_i/\epsilon_g$, and the probability simplex
$\Delta_g=\{w:w_i\geq0,\ \sum_{i\in g}w_i=1\}$.

\begin{proposition}[Credit allocation under a local-variation penalty]
\label{prop:variational_credit}
CAST uniquely minimizes
\begin{equation}
    w^{\mathrm{US}}
    =\arg\min_{w\in\Delta_g}
    \left\{
    \mathrm{KL}(w\,\|\,p^u)
    +\sum_{i\in g}w_i\log\rho_i
    \right\}.
    \label{eq:variational_credit}
\end{equation}
Moreover, the adaptive scale in Eq.~\eqref{eq:smoothness} bounds the relative
distortion due to smoothness alone:
\begin{equation}
    1\leq\rho_i\leq K_g+1,
    \qquad
    \frac{\max_{i\in g}s_i}{\min_{i\in g}s_i}\leq K_g+1.
    \label{eq:smoothness_bound}
\end{equation}
\end{proposition}

\begin{proof}
Let $Z_g=\sum_{i\in g}p_i^u/\rho_i$ and
$w_i^*=p_i^u/(\rho_i Z_g)=w_i^{\mathrm{US}}$. The objective satisfies
$F_g(w)=\mathrm{KL}(w\,\|\,w^*)-\log Z_g$, uniquely minimized at $w^*$
by nonnegativity of KL. Since
$0\leq\delta_i\leq\sum_j\delta_j\leq K_g\epsilon_g$, we have
$1\leq\rho_i\leq K_g+1$. Using $s_i=(\epsilon_g\rho_i)^{-1}$ gives
the smoothness-ratio bound.
\end{proof}

Equation~\eqref{eq:variational_credit} balances the uncertainty prior against
local teacher-score variation. Optimality concerns this allocation objective;
gradient and downstream effects are assessed in Figure~\ref{fig:grad_quality}
and Table~\ref{tab:credit_components}. Smoothness reweighting is bounded, but
uncertainty can still concentrate credit.

\subsection{Selective Token Anchoring and Optimization}
\label{sec:selective_kl}
\label{sec:objective}

The \emph{stay} component (Figure~\ref{fig:method_overview}, right)
constrains drift from shared parameter updates by
keeping unaligned predictions close to the initial student policy $\pi_0$
(the SFT initialization for cross-tokenizer pairs). We apply a selective KL penalty:
\begin{align}
    z_i &= \log\pi_0(a_i\mid x,a_{<i})
           -\log\pi_\theta(a_i\mid x,a_{<i}),
    \qquad d_i = \exp(z_i)-z_i-1,
    \label{eq:reference_ratio}
    \\[-0.3em]
    \mathcal L_{\rm anchor}
    &=\frac{\sum_{i\in\mathcal U}\min(d_i,d_{\max})}
    {\max(|\mathcal U|,1)}.
    \label{eq:anchor_loss}
\end{align}
Here $d_i\geq0$ vanishes when sampled-token probabilities match, and $d_{\max}$
clips outliers. Thus aligned positions remain free to absorb teacher credit,
while $\mathcal L_{\rm anchor}=0$ if $\mathcal U=\varnothing$.

We use multi-rollout sampling with $n$ independent student responses per prompt,
scoring and chunking each separately before averaging over their tokens. With
$r_i(\theta)=\pi_\theta(a_i\mid x,a_{<i})/
\pi_{\theta_{\rm old}}(a_i\mid x,a_{<i})$, the aligned-token objective is
\begin{equation}
\mathcal L_{\rm PG}=-\frac{1}{|\mathcal A|}\sum_{i\in\mathcal A}
\min\!\left(r_i\widehat A_i,
\operatorname{clip}(r_i,1-\varepsilon_{\rm low},1+\varepsilon_{\rm high})
\widehat A_i\right).
\label{eq:policy_loss}
\end{equation}
Unaligned positions have zero distillation advantage. The full objective is
\begin{equation}
    \mathcal L=\mathcal L_{\rm PG}+\lambda_{\rm KL}\mathcal L_{\rm reg}.
    \label{eq:overall_objective}
\end{equation}
For heterogeneous tokenizers, $\mathcal L_{\rm reg}=\mathcal L_{\rm anchor}$.
Shared-tokenizer training retains the same credit rule and applies reference
KL regularization toward $\pi_0$ over all response positions.

\section{Experiments}
\label{sec:experiments}

\setlength{\floatsep}{6pt plus 1pt minus 1pt}
\setlength{\textfloatsep}{10pt plus 2pt minus 1pt}

\subsection{Experimental Setup}
\label{sec:setup}

\paragraph{Baselines.}
Cross-tokenizer experiments distill Qwen3-8B into Mistral-7B and Llama-3.1-8B
\citep{grattafiori2024llama}, both initialized by SFT on
$400{,}000$ OpenThoughts examples \citep{guha2026openthoughts}.
Figure~\ref{fig:tokstats} characterizes tokenizer mismatch for these pairs.
Heterogeneous baselines include the student initialization, teacher-reward
GRPO, DPCA, minimum edit-distance alignment (MinED), Universal Logit
Distillation (ULD), Approximate Likelihood Matching (ALM), and a matched
uniform-credit control. DPCA is the decoded-text-alignment OPD baseline;
the uniform control splits each chunk advantage equally while retaining
CAST's clipping, selective anchoring, and two-rollout sampling to isolate
its \emph{shift} credit rule. Shared-tokenizer experiments pair
Qwen3-8B or Qwen3-30B-A3B teachers with Qwen3-1.7B and Qwen3-4B students,
using raw instruction models or the same 400K SFT initialization as marked
in each table, with GRPO, GKD, MiniLLM, and DistiLLM as baselines.
We also directly evaluate the official Hugging Face releases of Qwen3-8B,
Mistral-7B-Instruct-v0.3, and Llama-3.1-8B-Instruct without additional training.

\paragraph{Training data.}
We randomly sample both v1 and v2 from OpenThoughts using seed $44$.
The prompt-only v1 set contains $20{,}000$ prompts, predominantly competition
mathematics, with estimated shares of $78.4\%$ mathematics, $17.4\%$ code,
$\approx3.2\%$ science, and $\approx1.0\%$ commonsense or logic. Without domain
labels, we infer mathematics and code from the \texttt{\textbackslash boxed\{}
answer format and executable-Python-function template, respectively, and
estimate other domains manually. v2 contains $15{,}000$ OpenThoughts3-1.2M
prompts, with $5{,}000$ each
from mathematics, code, and science. Both sets use responses generated on-policy
by the current student and are reported separately.

\paragraph{Training and evaluation protocol.}
Hyperparameters (Table~\ref{tab:hparams}), compute budgets, checkpoint selection,
and the shared six-benchmark pass@1 evaluation protocol are detailed in
Appendix~\ref{app:hyperparams}.

\subsection{Main Results}
\label{sec:main}

\begin{table}[t]
\setlength{\abovecaptionskip}{4pt}
\caption{Cross-tokenizer pass@1 (\%). Avg.: six-benchmark macro mean (two decimals).
None: no additional training; SFT init: distillation initialization;
DPCA: original DPCA+OPD. \textbf{Bold} marks the best scores per student;
teacher scores are separate.}
\label{tab:main}
\centering
\footnotesize
\setlength{\tabcolsep}{3.7pt}
\begin{adjustbox}{max width=\linewidth}
\begin{tabular}{lllrrrrrrr}
\toprule
Student & Teacher & Method & A24 & A25 & A26 & M500 & GPQA & LCB & Avg. \\
\midrule
-- & Qwen3-8B & None & 69.5 & 56.6 & 60.4 & 81.0 & 57.8 & 59.0 & 64.05 \\
\midrule
\multirow{8}{*}{Mistral-7B}
 & -- & None & 0.2 & 0.0 & 0.0 & 13.8 & 23.0 & 4.6 & 6.93 \\
 & -- & SFT init & 41.5 & 36.5 & 48.3 & 74.1 & 37.6 & \best{28.0} & 44.33 \\
 & Qwen3-8B & GRPO     & 19.4 & 23.1 & 25.4 & 65.5 & 31.2 & 21.2 & 30.97 \\
 & Qwen3-8B & MinED    & 44.7 & 36.8 & 47.9 & 74.4 & 34.6 & 25.7 & 44.02 \\
 & Qwen3-8B & ULD      & 45.5 & 36.8 & 47.3 & \best{79.6} & 37.2 & 22.2 & 44.77 \\
 & Qwen3-8B & ALM      & 22.3 & 18.9 & 26.2 & 63.5 & 15.7 & 12.5 & 26.52 \\
 & Qwen3-8B & DPCA     & 41.8 & 35.8 & 47.7 & 73.8 & 35.4 & 13.9 & 41.40 \\
 & Qwen3-8B & \textbf{CAST} & \best{45.7} & \best{40.9} & \best{51.0} & 78.5 & \best{39.5} & 27.7 & \best{47.22} \\
\midrule
\multirow{8}{*}{Llama-3.1-8B}
 & -- & None & 3.3 & 0.4 & 1.7 & 45.5 & 26.5 & 9.8 & 14.53 \\
 & -- & SFT init & 39.0 & 35.8 & 39.2 & 73.9 & \best{37.4} & \best{27.2} & 42.08 \\
 & Qwen3-8B & GRPO     & 40.9 & 35.2 & 40.2 & 73.2 & 37.2 & 26.2 & 42.15 \\
 & Qwen3-8B & MinED    & 41.0 & 36.8 & 39.4 & 74.9 & 32.7 & 23.6 & 41.40 \\
 & Qwen3-8B & ULD      & 41.6 & 36.2 & 41.9 & 75.1 & 34.7 & 25.8 & 42.55 \\
 & Qwen3-8B & ALM      & 35.6 & 29.3 & 40.8 & \best{76.2} & 36.9 & 26.4 & 40.87 \\
 & Qwen3-8B & DPCA     & 39.7 & 37.6 & 41.9 & 73.9 & 33.3 & 24.1 & 41.75 \\
 & Qwen3-8B & \textbf{CAST} & \best{44.1} & \best{37.7} & \best{45.0} & 75.0 & 34.5 & 26.9 & \best{43.87} \\
\bottomrule
\end{tabular}
\end{adjustbox}
\end{table}

\begin{table}[t]
\setlength{\abovecaptionskip}{4pt}
\caption{Shared-tokenizer pass@1 (\%) with raw Qwen3-1.7B, a Qwen3-8B teacher,
and v1.}
\label{tab:ladder}
\centering
\footnotesize
\setlength{\tabcolsep}{4.2pt}
\begin{adjustbox}{max width=\linewidth}
\begin{tabular}{lllrrrrrrr}
\toprule
Student & Teacher & Method & A24 & A25 & A26 & M500 & GPQA & LCB & Avg. \\
\midrule
-- & Qwen3-8B & None & 69.5 & 56.6 & 60.4 & 81.0 & 57.8 & 59.0 & 64.05 \\
\midrule
Qwen3-1.7B & -- & None & 46.2 & 34.0 & 42.5 & 75.9 & 37.8 & 34.1 & 45.08 \\
Qwen3-1.7B & Qwen3-8B & GRPO & 44.5 & 33.2 & 41.2 & 75.5 & 36.4 & 35.5 & 44.38 \\
Qwen3-1.7B & Qwen3-8B & GKD & 50.3 & 35.4 & 42.7 & 76.2 & 36.6 & \best{35.6} & 46.13 \\
Qwen3-1.7B & Qwen3-8B & MiniLLM & 48.5 & 35.1 & 46.7 & 76.9 & 35.0 & 34.0 & 46.03 \\
Qwen3-1.7B & Qwen3-8B & DistiLLM & 49.4 & \best{39.1} & 44.2 & 76.4 & 37.4 & 33.0 & 46.58 \\
Qwen3-1.7B & Qwen3-8B & \textbf{CAST} & \best{51.6} & 37.0 & \best{51.7} & \best{77.2} & \best{40.0} & 34.6 & \best{48.68} \\
\bottomrule
\end{tabular}
\end{adjustbox}
\end{table}

\begin{table}[t]
\setlength{\abovecaptionskip}{4pt}
\caption{Shared-tokenizer CAST ablations (Qwen3-8B $\rightarrow$ raw Qwen3-1.7B,
v1). Uniform uses equal weights; shuffled permutes CAST weights within chunks;
fixed $K{=}6$ replaces punctuation-defined chunks.}
\label{tab:shared_ablation}
\centering
\footnotesize
\setlength{\tabcolsep}{5.0pt}
\begin{tabular}{lrrrrrrr}
\toprule
Setting & A24 & A25 & A26 & M500 & GPQA & LCB & Avg. \\
\midrule
Uniform credit & 46.6 & 36.7 & 46.5 & 76.5 & 38.3 & 33.9 & 46.42 \\
Fixed chunks ($K{=}6$) & 48.9 & 36.6 & 46.5 & 76.3 & 39.3 & \best{35.1} & 47.12 \\
Shuffled credit & 48.2 & 36.6 & 42.9 & 76.2 & 36.0 & 33.0 & 45.48 \\
\textbf{CAST} & \best{51.6} & \best{37.0} & \best{51.7} & \best{77.2} & \best{40.0} & 34.6 & \best{48.68} \\
\bottomrule
\end{tabular}
\end{table}

\begin{table}[t]
\setlength{\abovecaptionskip}{4pt}
\caption{Mistral component ablations (Qwen3-8B teacher). Credit columns identify
weighting signals; Sel./Global anchor unaligned/all positions. Shuffled
permutes CAST weights within each chunk; rollouts count responses per prompt.}
\label{tab:credit_components}
\centering
\footnotesize
\setlength{\tabcolsep}{2.4pt}
\begin{adjustbox}{max width=\linewidth}
\begin{tabular}{lcccccccrrrrrrr}
\toprule
\multirow{2}{*}{Setting} &
\multicolumn{3}{c}{Credit assignment} &
\multirow{2}{*}{\shortstack{KL\\anchor}} &
\multicolumn{3}{c}{Rollouts} &
\multicolumn{7}{c}{Pass@1 (\%)} \\
\cmidrule(lr){2-4}\cmidrule(lr){6-8}\cmidrule(lr){9-15}
& Unif. & Uncert. & Smooth. & & $1$ & $2$ & $4$ &
A24 & A25 & A26 & M500 & GPQA & LCB & Avg. \\
\midrule
DPCA & $\checkmark$ &  &  & None & $\checkmark$ &  &  &
41.8 & 35.8 & 47.7 & 73.8 & 35.4 & 13.9 & 41.40 \\
\midrule
Uniform credit & $\checkmark$ &  &  & Sel. &  & $\checkmark$ &  &
42.3 & 37.2 & 46.7 & 78.9 & 38.5 & 26.8 & 45.07 \\
Shuffled credit &  & $\checkmark$ & $\checkmark$ & Sel. &  & $\checkmark$ &  &
45.5 & 36.4 & 46.2 & 75.5 & 33.5 & 27.9 & 44.17 \\
Uncertainty only &  & $\checkmark$ &  & Sel. &  & $\checkmark$ &  &
\best{48.3} & \best{41.6} & 51.2 & 75.3 & \best{40.2} & 17.0 & 45.60 \\
Smoothness only &  &  & $\checkmark$ & Sel. &  & $\checkmark$ &  &
45.7 & 35.8 & 46.9 & 73.5 & 37.1 & 21.1 & 43.35 \\
No anchoring &  & $\checkmark$ & $\checkmark$ & None &  & $\checkmark$ &  &
47.5 & 37.1 & 52.5 & 73.8 & 35.6 & 17.5 & 44.00 \\
Global anchoring &  & $\checkmark$ & $\checkmark$ & Global &  & $\checkmark$ &  &
46.1 & 37.2 & 49.4 & 75.6 & 37.6 & 19.9 & 44.30 \\
\midrule
No anchor + single &  & $\checkmark$ & $\checkmark$ & None & $\checkmark$ &  &  &
46.4 & 39.9 & 46.9 & 77.9 & 37.9 & 27.0 & 46.00 \\
Single rollout &  & $\checkmark$ & $\checkmark$ & Sel. & $\checkmark$ &  &  &
45.2 & 35.7 & 48.5 & \best{81.5} & 37.1 & \best{28.4} & 46.07 \\
Four rollouts &  & $\checkmark$ & $\checkmark$ & Sel. &  &  & $\checkmark$ &
47.2 & 37.2 & \best{52.7} & 75.2 & 38.1 & 19.9 & 45.05 \\
\textbf{CAST} &  & $\checkmark$ & $\checkmark$ & Sel. &  & $\checkmark$ &  &
45.7 & 40.9 & 51.0 & 78.5 & 39.5 & 27.7 & \best{47.22} \\
\bottomrule
\end{tabular}
\end{adjustbox}
\end{table}

\begin{table}[t]
\setlength{\abovecaptionskip}{4pt}
\caption{Three-seed Mistral pass@1 (\%, lr $10^{-6}$): mean $\pm$ sample SD
of highest-Avg-6 evaluated checkpoints (one per seed).
Per-seed results and selection: Appendix~\ref{app:seed_robustness}.}
\label{tab:seed_summary}
\centering
\footnotesize
\setlength{\tabcolsep}{2pt}
\begin{tabular*}{\linewidth}{@{\extracolsep{\fill}}lrrrrrrr@{}}
\toprule
Method & A24 & A25 & A26 & M500 & GPQA & LCB & Avg. \\
\midrule
DPCA & $40.4\pm1.8$ & $35.8\pm1.3$ & $45.7\pm1.8$ & $73.8\pm0.1$ & $35.1\pm0.9$ & $13.2\pm3.2$ & $40.66\pm0.86$ \\
Uniform credit & $44.4\pm1.8$ & $36.1\pm1.0$ & $46.2\pm1.2$ & $77.8\pm1.2$ & $37.2\pm1.2$ & $27.2\pm1.0$ & $44.81\pm0.22$ \\
Uncertainty only & $49.0\pm1.5$ & $40.7\pm1.0$ & $49.7\pm1.3$ & $74.7\pm1.1$ & $39.3\pm0.8$ & $17.8\pm0.8$ & $45.21\pm0.64$ \\
\textbf{CAST} & $45.6\pm0.7$ & $39.9\pm0.9$ & $50.1\pm1.2$ & $79.0\pm1.2$ & $39.5\pm0.7$ & $27.9\pm0.2$ & $47.01\pm0.28$ \\
\bottomrule
\end{tabular*}

\par\vspace{\floatsep}
\caption{Three-seed Llama-3.1-8B pass@1 (\%, Qwen3-8B teacher): mean $\pm$
sample SD. Per-seed checkpoints and aggregation: Appendix~\ref{app:seed_robustness}.}
\label{tab:llama_seed_summary}
\centering
\footnotesize
\setlength{\tabcolsep}{2pt}
\begin{tabular*}{\linewidth}{@{\extracolsep{\fill}}lrrrrrrr@{}}
\toprule
Method & A24 & A25 & A26 & M500 & GPQA & LCB & Avg. \\
\midrule
ULD & $41.4\pm0.3$ & $34.9\pm1.1$ & $41.3\pm1.0$ & $74.6\pm0.8$ & $34.5\pm1.3$ & $25.7\pm0.1$ & $42.08\pm0.43$ \\
\textbf{CAST} & $43.3\pm0.7$ & $38.0\pm0.9$ & $43.5\pm1.7$ & $74.8\pm0.6$ & $36.0\pm1.3$ & $26.7\pm0.2$ & $43.72\pm0.16$ \\
\bottomrule
\end{tabular*}
\end{table}

\begin{figure}[t]
    \centering
    \includegraphics[width=\textwidth]{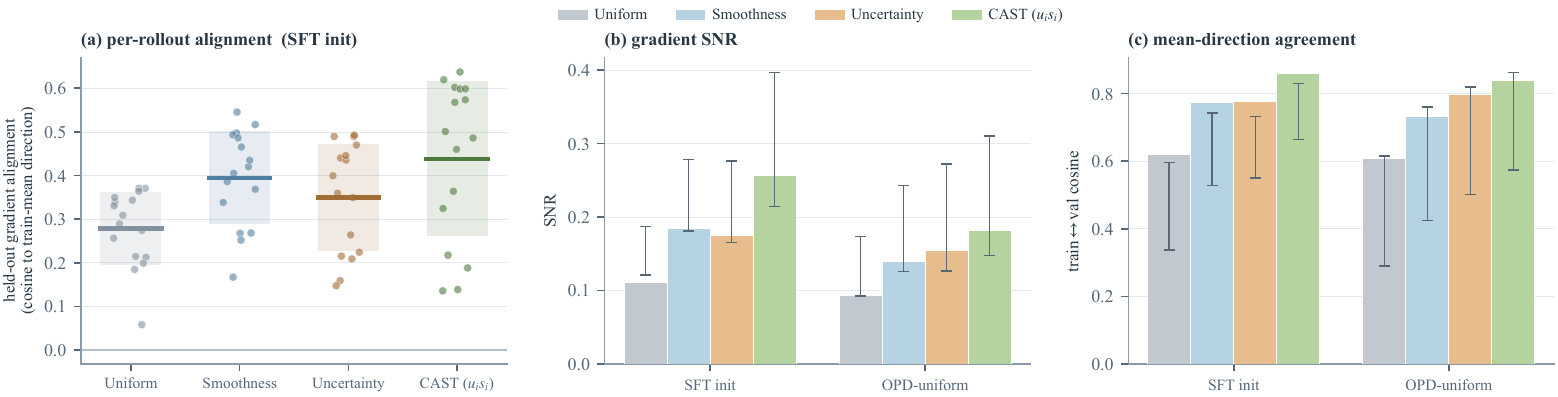}
    \caption{\textbf{Gradient consistency across four credit rules.}
Sketched Mistral LM-head gradients use 48 shared rollouts per checkpoint
(32 train, 16 held out). \textbf{(a)} Held-out cosine with each rule's training
mean at SFT initialization (mean $\pm$ SD). \textbf{(b)} SNR.
\textbf{(c)} Training--held-out mean cosine. Error bars in b,c: $95\%$ bootstrap CIs.
Protocol: Appendix~\ref{app:gradient_diagnostics}.}
    \label{fig:grad_quality}
\end{figure}

\begin{figure}[t]
\centering
\includegraphics[width=\textwidth]{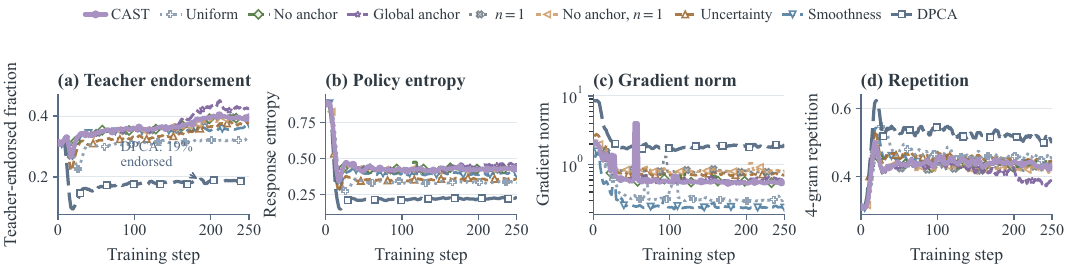}
\caption{\textbf{Mistral training dynamics.}
\textbf{(a)} Positive-teacher-advantage fraction; \textbf{(b)} policy entropy;
\textbf{(c)} gradient norm; \textbf{(d)} $4$-gram repetition.
Curves: five-step moving averages. Details: Appendix~\ref{app:training_dynamics}.}
\label{fig:training_dynamics}
\end{figure}

\paragraph{Cross-tokenizer distillation.}
Table~\ref{tab:main} shows that CAST achieved the
highest six-benchmark average for both student families: $47.22$ for Mistral
and $43.87$ for Llama, improving over ULD by $2.45$ and $1.32$ points,
respectively. CAST led ULD on five benchmarks for Mistral and four for Llama;
MATH-500 was lower for both, as was GPQA for Llama. Relative to the SFT
initialization, CAST improved five of six Mistral scores and raised the
average from $44.33$ to $47.22$, while LCB declined slightly from $28.0$
to $27.7$. Llama
improved on all three AIME benchmarks and MATH-500, while GPQA and LCB declined;
its average increased from $42.08$ to $43.87$.
Original DPCA reached averages of $41.40$ and $41.75$, respectively.

\paragraph{Shared-tokenizer distillation.}
With Qwen3-8B as teacher and v1 prompts, CAST improves all six
benchmarks over raw Qwen3-1.7B, reaching $48.68$ Avg-6 versus $46.58$ for the
strongest baseline, DistiLLM (Table~\ref{tab:ladder}). It exceeds uniform and
shuffled credit by $2.26$ and $3.20$ points (Table~\ref{tab:shared_ablation}).
Shuffling preserves chunk weights but changes their token assignments,
supporting the importance of credit placement. Punctuation-defined chunks
outperform fixed-length chunks ($K{=}6$) by $1.56$ points, although the latter
score higher on LCB ($35.1$ versus $34.6$). These controls support CAST's
\emph{shift} rule under exact token correspondence. Broader comparisons show
that its gains vary with teacher--student scale and prompt composition
(Tables~\ref{tab:scale_grid} and~\ref{tab:domain} in Appendix~\ref{app:scaling_data}).

\subsection{Discussions}
\label{sec:discussion}
\label{sec:ablations}

\paragraph{Credit assignment beyond chunk-level alignment.}
In Table~\ref{tab:credit_components}, the matched uniform control retains
decoded-text alignment, selective anchoring, and two-rollout sampling but
removes structured token credit, reducing the average from $47.22$ to $45.07$.
Shuffling preserves CAST's within-chunk weights but breaks their token
correspondence; CAST improves five of six benchmarks over this control, raising the
average from $44.17$ to $47.22$, with gains of $4.8$ on AIME-2026 and $6.0$ on
GPQA, while LCB is slightly lower ($27.7$ versus $27.9$).
These controls support CAST's token-specific \emph{shift} credit.
Appendix~\ref{app:token_credit} illustrates its allocation within a chunk and
across token categories in one rollout (Figure~\ref{fig:credit_case}).

\paragraph{Performance across training seeds.}
Tables~\ref{tab:seed_summary} and~\ref{tab:llama_seed_summary} summarize three
seeds per configuration. CAST reaches $47.01\pm0.28$ Avg-6 on Mistral and
$43.72\pm0.16$ on Llama, versus $42.08\pm0.43$ for Llama ULD.
For both students, all three CAST scores exceed the corresponding comparators'
highest Avg-6 (Tables~\ref{tab:seed_robustness}--\ref{tab:llama_seed_records} in
Appendix~\ref{app:seed_robustness}).

\paragraph{CAST yields more consistent gradient directions.}
Figure~\ref{fig:grad_quality} compares four credit rules on shared rollouts at
SFT initialization and an OPD-uniform checkpoint. CAST has the highest
gradient signal-to-noise ratio (SNR) and training--held-out mean-direction
agreement at both; SFT SNR rises from $0.111$ under uniform credit to $0.256$
under CAST (protocol in Appendix~\ref{app:gradient_diagnostics}).
Complementary training traces in Figure~\ref{fig:training_dynamics} show higher
teacher endorsement and policy entropy, smaller gradient norms, and lower
repetition than DPCA over the final $50$ updates.
Appendix~\ref{app:training_dynamics} details these runs; selected 2024 AIME~I
responses in Appendix~\ref{app:case_study} illustrate CAST reaching final answers
where SFT or DPCA responses repeat or remain unfinished.

\paragraph{Smoothness helps mitigate training collapse.}
Smoothness produces a more balanced domain profile by moderating
uncertainty-driven credit: with selective anchoring and two rollouts fixed,
LCB rises from $17.0$ to $27.7$ and Avg-6 from $45.60$ to $47.22$, while the
other five benchmarks change by $-0.20$ points on average
(Table~\ref{tab:credit_components}). This moderation discounts uncertain tokens
when teacher scores vary sharply, redistributing credit toward more locally
regular supervision. It primarily uses teacher structure: $79\%$ of
redistribution occurs at teacher-variation positions, and comparable variation
persists without projection under shared tokenization
(Figure~\ref{fig:smoothness_attribution} in Appendix~\ref{app:smoothness_attribution}).
Separate seed-$44$ diagnostics show slightly broader effective credit support
for CAST ($2.814$ tokens versus $2.632$ for uncertainty-only weighting;
Figure~\ref{fig:credit_concentration}, Appendix~\ref{app:seed_robustness}).
This broader, teacher-informed allocation offers a plausible explanation for
CAST's fewer observed collapses than uncertainty-only credit across six seeds
(Table~\ref{tab:collapse}, Appendix~\ref{app:seed_robustness}).

\paragraph{The role of selective anchoring.}
With CAST credit and $n{=}2$ fixed, selective anchoring reaches $47.22$ Avg-6,
versus $44.30$ with global anchoring and $44.00$ without it; the largest
selective--global gain is on LCB ($+7.8$; Table~\ref{tab:credit_components}).
With selective anchoring, two rollouts outperform one ($+1.15$ Avg-6) and four
(Appendix~\ref{app:sampling}). Over the final $50$ updates, global anchoring has higher entropy and
lower repetition than CAST, yet weaker downstream performance
(Figure~\ref{fig:training_dynamics}b,d),
consistent with overconstraint: global anchoring penalizes departures from SFT
at aligned positions, potentially opposing teacher credit. Selective anchoring
instead constrains drift at unaligned positions while leaving aligned positions
to teacher supervision.

\section{Conclusion}
\label{sec:conclusion}

We introduced CAST for token-level \emph{shift-or-stay} cross-tokenizer OPD:
uncertainty--smoothness credit redistributes aligned chunk advantages, while
selective anchoring keeps unaligned predictions close to the initial policy.
The advantage-preserving credit rule admits an exact KL-regularized
characterization and extends to shared-tokenizer punctuation-defined or
fixed-length chunks. CAST improves the six-benchmark average over the strongest
evaluated baseline by $1.32$--$2.45$ points for heterogeneous students
and $2.10$ for Qwen3-1.7B. Controls support token-specific credit placement,
more consistent LM-head gradients, and gains across Mistral and Llama seeds.
Smoothness recovers code performance and broadens effective credit support,
alongside fewer observed collapses; selective anchoring improves the
two-rollout configuration.

\section*{Limitations}

Our evaluation covers a limited set of model families and mainly
reasoning-oriented training data, so domain gains may change with student
capacity or prompt composition. Training compute varies across methods, and
repeated-seed performance covers four Mistral and two Llama configurations.
Broader model and data coverage,
matched-compute evaluation, and direct tests of how credit allocation affects
full-model optimization and training collapse remain future work.

\par
\FloatBarrier
\setlength{\floatsep}{12pt plus 2pt minus 1pt}
\setlength{\textfloatsep}{14pt plus 2pt minus 2pt}
\label{end:main_text}
\typeout{CAST_MAIN_TEXT_LAST_PAGE=\thepage}
\clearpage
\subsubsection*{Reproducibility Statement}
The method is described in \S\ref{sec:method}, with credit weights in
Eq.~\eqref{eq:us_credit}, selective anchoring
in Eqs.~\eqref{eq:reference_ratio}--\eqref{eq:anchor_loss}, and the training
objective in Eq.~\eqref{eq:overall_objective}. The experimental setup is in
\S\ref{sec:setup}; training and evaluation protocols are in
Appendix~\ref{app:hyperparams}. Tables~\ref{tab:seed_summary} and~\ref{tab:llama_seed_summary} report cross-seed
statistics, with individual checkpoints and collapse records in
Appendix~\ref{app:seed_robustness}. Protocols for the main-text gradient and
training diagnostics (Figures~\ref{fig:grad_quality} and~\ref{fig:training_dynamics})
are in Appendices~\ref{app:gradient_diagnostics} and~\ref{app:training_dynamics}.
The token-level example is provided in Appendix~\ref{app:token_credit}.
Figure-generation scripts use reported results, recorded token-level traces,
gradient summaries, and training logs.

\section*{AI Use Disclosure}
Generative AI tools were used for manuscript polishing.

\bibliography{iclr2027_conference}
\bibliographystyle{iclr2027_conference}

\clearpage
\appendix

\section{Additional Method Details}

\label{app:method_details}

\subsection{Chunk Construction}
\label{app:chunk_construction}

\paragraph{Heterogeneous tokenizers.}
We use DPCA's decoded-text chunk alignment \citep{niu2026tokenizerbarrier}:
a greedy two-pointer scan extends the shorter side until the decoded student
and teacher strings match after Unicode and tokenizer-specific whitespace
normalization. Each matched pair of contiguous token groups $(I_c,J_c)$
forms an \emph{alignment chunk}. We discard excessively large or unreliable
matches; retained student positions form $\mathcal A$ and the remainder
$\mathcal U$. Each retained alignment chunk's teacher score is projected onto
the student grid as
\begin{equation}
    \widetilde{\ell}_i^T
    =
    \frac{1}{|I_c|}\sum_{j\in J_c}\ell_j^T,
    \qquad i\in I_c.
    \label{eq:teacher_projection}
\end{equation}
After the projection in Eq.~\eqref{eq:teacher_projection}, we regroup contiguous
aligned student tokens into fixed-length $K$-token \emph{credit chunks} (plus a
shorter residual chunk), collected in $\mathcal G$. These define the units for
advantage computation and token-level credit allocation; their boundaries need
not coincide with alignment-chunk boundaries. We use $K=6$ unless stated otherwise.

\paragraph{Shared tokenizer.}
The training path uses identity correspondence between response tokens, so
$M=N$, $\widetilde{\ell}_i^T=\ell_i^T$, and $\mathcal U=\varnothing$.
We split responses after commas, periods, exclamation marks, and question marks
(with any trailing suffix as the final chunk); these punctuation-defined credit chunks
directly form $\mathcal G$ without fixed-$K$ regrouping.
The fixed-length variant in Table~\ref{tab:shared_ablation} instead partitions
the response into consecutive $K=6$ token chunks, with a shorter final chunk
when needed.

\FloatBarrier

\section{Additional Experiments and Diagnostics}

\label{app:additional_experiments}

\subsection{Teacher Scale and Training Distribution}

\label{app:scaling_data}

\begin{table}[!htbp]
\caption{Shared-tokenizer teacher--student scaling results from raw Qwen3
instruction models. The method column identifies results trained with
CAST.}
\label{tab:scale_grid}
\centering
\footnotesize
\setlength{\tabcolsep}{4.1pt}
\begin{tabular}{lllrrrrrrr}
\toprule
Student & Teacher & Method & A24 & A25 & A26 & M500 & GPQA & LCB & Avg. \\
\midrule
\multirow{3}{*}{Qwen3-1.7B} & None (raw) & -- & 46.2 & 34.0 & 42.5 & 75.9 & 37.8 & 34.1 & 45.08 \\
 & Qwen3-8B & \textbf{CAST} & \best{51.6} & 37.0 & \best{51.7} & \best{77.2} & \best{40.0} & 34.6 & \best{48.68} \\
 & Qwen3-30B-A3B & \textbf{CAST} & 50.2 & \best{39.4} & 49.0 & 76.5 & 37.9 & \best{35.3} & 48.05 \\
\midrule
\multirow{3}{*}{Qwen3-4B} & None (raw) & -- & 66.5 & 58.6 & 58.8 & \best{81.7} & \best{54.5} & \best{54.4} & 62.42 \\
 & Qwen3-8B & \textbf{CAST} & 66.4 & 59.3 & 60.6 & 81.2 & 54.3 & 52.2 & 62.33 \\
 & Qwen3-30B-A3B & \textbf{CAST} & \best{69.2} & \best{63.6} & \best{66.9} & 81.5 & 53.7 & 52.6 & \best{64.58} \\
\bottomrule
\end{tabular}
\end{table}

\paragraph{Teacher--student scaling.}
Table~\ref{tab:scale_grid} shows that the benefit of a larger teacher depends
on student capacity. For Qwen3-1.7B, the 8B teacher reaches $48.68$, above the
30B-A3B teacher's $48.05$; both exceed the raw student's $45.08$. For Qwen3-4B,
the 8B teacher leaves the average nearly unchanged ($62.33$ versus $62.42$),
whereas the 30B-A3B teacher reaches $64.58$ and improves all three AIME
benchmarks. Thus the larger teacher benefits the 4B student but provides no
additional average gain for the 1.7B student in these comparisons.

\begin{table}[!htbp]
\caption{Student, teacher, and prompt-set comparison from raw Qwen3 instruction
models. All trained rows use CAST with $n{=}2$.}
\label{tab:domain}
\centering
\footnotesize
\setlength{\tabcolsep}{2.0pt}
\begin{adjustbox}{max width=\linewidth}
\begin{tabular}{llllrrrrrrr}
\toprule
Student & Teacher & Data & Method & A24 & A25 & A26 & M500 & GPQA & LCB & Avg. \\
\midrule
\multirow{3}{*}{Qwen3-1.7B} & -- & -- & None & 46.2 & 34.0 & 42.5 & 75.9 & 37.8 & 34.1 & 45.08 \\
 & Qwen3-8B & v1 & CAST & \best{51.6} & 37.0 & \best{51.7} & \best{77.2} & \best{40.0} & 34.6 & \best{48.68} \\
 & Qwen3-8B & v2 & CAST & 49.1 & 37.9 & 48.8 & 76.8 & 37.9 & 34.4 & 47.48 \\
\midrule
\multirow{5}{*}{Qwen3-4B} & -- & -- & None & 66.5 & 58.6 & 58.8 & 81.7 & 54.5 & \best{54.4} & 62.42 \\
 & Qwen3-8B & v1 & CAST & 66.4 & 59.3 & 60.6 & 81.2 & 54.3 & 52.2 & 62.33 \\
 & Qwen3-8B & v2 & CAST & 67.7 & 62.8 & 63.3 & \best{82.1} & 54.5 & 53.9 & 64.05 \\
 & Qwen3-30B-A3B & v1 & CAST & 69.2 & \best{63.6} & \best{66.9} & 81.5 & 53.7 & 52.6 & 64.58 \\
 & Qwen3-30B-A3B & v2 & CAST & \best{70.3} & 62.6 & \best{66.9} & 81.0 & \best{55.4} & 53.4 & \best{64.93} \\
\bottomrule
\end{tabular}
\end{adjustbox}
\end{table}

\paragraph{Training distribution and teacher scale.}
Table~\ref{tab:domain} compares prompt sets and teacher scale across raw
Qwen3 students. For the 1.7B student with an 8B teacher, CAST reaches averages
of $48.68$ on v1 and $47.48$ on v2. For the 4B student, v2 raises the average
from $62.33$ to $64.05$ with the 8B teacher and from $64.58$ to $64.93$ with the
30B-A3B teacher. The larger teacher obtains the higher 4B average under both
prompt sets, with the gains concentrated on the three AIME benchmarks.

Prompt distribution further affects code retention in the shared-tokenizer
setting (Table~\ref{tab:domain}). With an 8B teacher, Qwen3-1.7B retains LCB
near the raw student's $34.1$, reaching $34.6$/$34.4$ on v1/v2. For Qwen3-4B,
v2 reduces the LCB loss relative to the raw model under both teachers:
scores increase from $52.2$ to $53.9$ with the 8B teacher and from $52.6$ to
$53.4$ with the 30B-A3B teacher, alongside higher Avg-6. These results locate
the observed trade-offs in both credit design and prompt distribution.

\FloatBarrier

\subsection{Additional Sampling Analysis}

\label{app:sampling}

\paragraph{Sampling multiplicity and selective anchoring.}
Table~\ref{tab:credit_components} compares $n\in\{1,2,4\}$ independent responses
per prompt, generating 64, 128, or 256 responses per update at a fixed batch
of 64 prompts. At the reported checkpoints, with selective anchoring fixed,
two rollouts improve four of six benchmarks over one, including gains of
$0.5$, $5.2$, and $2.5$ points on AIME-2024, AIME-2025, and AIME-2026,
and raise Avg-6 from $46.07$ to $47.22$. Four rollouts achieve the strongest
AIME-2024 and AIME-2026 scores among these settings, but reduce LCB by $7.8$
points relative to $n{=}2$ and lower Avg-6 to $45.05$. Without anchoring,
increasing from one to two rollouts instead lowers Avg-6 from $46.00$ to
$44.00$. Thus, selective anchoring contributes $+3.22$ points at $n{=}2$
but only $+0.07$ at $n{=}1$, a difference of $3.15$ points in these estimates.
Anchor scope also matters: at $n{=}2$, selective anchoring improves five of
six benchmarks over global anchoring ($47.22$ versus $44.30$ Avg-6), with
the largest gap on LCB ($27.7$ versus $19.9$). These comparisons favor CAST's
two-rollout configuration with selective anchoring among the tested settings;
the benefit of additional responses depends on the accompanying regularization.

\subsection{Token-Level Credit Analysis}

\label{app:token_credit}

\begin{figure}[!htbp]
\centering
\includegraphics[width=0.95\textwidth]{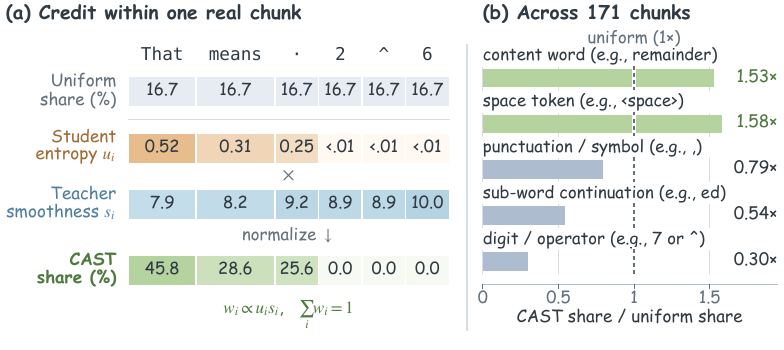}
\caption{\textbf{CAST redistributes credit to individual tokens.}
\textbf{(a)} Uniform shares, uncertainty $u_i$, smoothness $s_i$, and CAST
weights $w_i$ for a six-token chunk from a Qwen3-8B $\rightarrow$ Mistral-7B rollout.
\textbf{(b)} CAST-to-uniform credit-share ratios by token category across
all $171$ chunks of the same rollout; ratios above $1$ indicate upweighting.}
\label{fig:credit_case}
\end{figure}

\paragraph{Token-level analysis of credit assignment.}
Figure~\ref{fig:credit_case} illustrates the credit rule in
Eq.~\eqref{eq:us_credit} on a recorded training rollout.
In panel (a), uniform credit assigns $16.7\%$ to each token, whereas CAST
assigns $45.8\%$, $28.6\%$, and $25.6\%$ to the three high-entropy positions.
The near-deterministic digit and operator tokens in $2^6$ receive effectively
no share. Panel (b) aggregates the complete rollout: content words and
space tokens receive $1.53\times$ and $1.58\times$ their uniform shares, while
punctuation, sub-word continuations, and digit/operator tokens receive
$0.79\times$, $0.54\times$, and $0.30\times$. The controlled downstream
comparison is in Table~\ref{tab:credit_components}.

\FloatBarrier

\subsection{Training Dynamics: Protocol and Analysis}

\label{app:training_dynamics}

\paragraph{Diagnostic protocol.}
The main-text training curves in Figure~\ref{fig:training_dynamics} compare
Qwen3-8B $\rightarrow$ Mistral-7B CAST with seven pipeline variants and DPCA.
Global anchor regularizes all
response positions; No anchor removes reference regularization. Pipeline
variants use two rollouts unless marked $n{=}1$. DPCA uses the plain OPD loss
with uniform within-chunk credit and one rollout. The Uniform diagnostic uses
learning rate $5\times10^{-7}$; other runs use $10^{-6}$. These runs are separate
from the seed evaluation in Table~\ref{tab:seed_summary}.
The teacher-endorsed fraction is the fraction of rollout tokens for which the
teacher assigns a higher log-probability than the student. The other panels
show mean per-token response-policy entropy, gradient norm on a log scale,
and rollout $4$-gram repetition. Curves show five-step moving averages.

\paragraph{Training dynamics and stability.}
Credit concentration, averaged over training, shows
that CAST assigns each six-token chunk's budget to an effective $2.8$ tokens,
whereas smoothness-only credit remains diffuse over $5.3$ tokens. This agrees
with the uncertainty-sensitive token-level behavior in
Figure~\ref{fig:credit_case}. The training curves show how these configurations
evolve. Over the final 50 updates, DPCA's teacher-endorsed fraction falls to
$0.19$, compared with $0.40$ for CAST, so the teacher would downweight roughly
$80\%$ of the sampled tokens (Figure~\ref{fig:training_dynamics}a). At the same
time, DPCA reaches a response entropy of $0.22$ versus $0.43$ for CAST, a mean
gradient norm of $1.86$ versus $0.55$, and a $4$-gram repetition rate of $0.51$
versus $0.44$. Together, these trends describe a low-entropy policy whose
sampled tokens often receive lower teacher scores, with larger gradients and
more repetitive generation. The uniform-credit
diagnostic run lies between CAST and DPCA in teacher endorsement,
response entropy, and repetition, while having a smaller gradient norm.

The added anchor controls show that these diagnostics need to be interpreted
alongside downstream performance. Over the same final 50 updates, global
anchoring reaches teacher endorsement $0.43$, response entropy $0.45$, and
repetition $0.39$, compared with $0.40$, $0.43$, and $0.44$ for CAST. Despite
these favorable entropy and repetition statistics, global anchoring has a lower
reported six-benchmark average ($44.30$ versus $47.22$;
Table~\ref{tab:credit_components}). With one
rollout, removing anchoring leaves endorsement, entropy, and repetition close
to the selectively anchored $n{=}1$ run, while increasing the mean gradient
norm from $0.79$ to $0.85$. Together with the anchor--rollout comparison in
Table~\ref{tab:credit_components}, these curves show that anchoring affects
optimization behavior and downstream transfer differently across settings.

\FloatBarrier

\subsection{Gradient Diagnostic Protocol}
\label{app:gradient_diagnostics}

Figure~\ref{fig:grad_quality} compares four credit rules at the Mistral-7B SFT
initialization and an OPD-uniform (DPCA) checkpoint at step 150; neither policy
was trained with CAST's product rule. At each checkpoint, the rules reuse the
same 48 on-policy rollouts, generated with the training response limit of
16,384 tokens. We normalize weights proportional to $1$, $s_i$, $u_i$, or
$u_i s_i$ within each chunk and form token advantages $A_gw_i$. We measure
the resulting policy gradients on the language-model head and compress them
with a fixed CountSketch shared by all rules.

Let $g$ denote a sketched per-rollout head gradient. Gradient SNR is
$\|\mathbb{E}[g]\|^2/\mathbb{E}\|g-\mathbb{E}[g]\|^2$, which measures the
squared mean-gradient magnitude relative to variation across rollouts.
For direction comparisons, we split the rollouts into 32 for mean-gradient
estimation and 16 held out. Figure~\ref{fig:grad_quality}a shows each held-out
rollout gradient's cosine with its rule's training-set mean gradient;
Figure~\ref{fig:grad_quality}c compares the training- and held-out-set mean
gradients. Bars in panels b and c show point estimates with $95\%$ bootstrap
intervals from 2,000 resamples. Each rule is compared with its own mean
direction, so these metrics measure consistency across sampled rollouts.

CAST has the highest SNR and training--held-out mean-gradient cosine at both
checkpoints, with the single-factor rules between CAST and uniform credit.
Relative to uniform credit, SNR increases from $0.111$ to $0.256$ at SFT
initialization and from $0.093$ to $0.182$ at the DPCA checkpoint. At SFT
initialization, mean held-out rollout alignment also increases from $0.279$
to $0.438$. These observations concern sampled head gradients at the two
fixed policies; downstream effects are assessed by the credit-rule controls
in Table~\ref{tab:credit_components}.

\FloatBarrier

\subsection{Smoothness Attribution}
\label{app:smoothness_attribution}

\begin{figure}[!htbp]
    \centering
    \includegraphics[width=\textwidth]{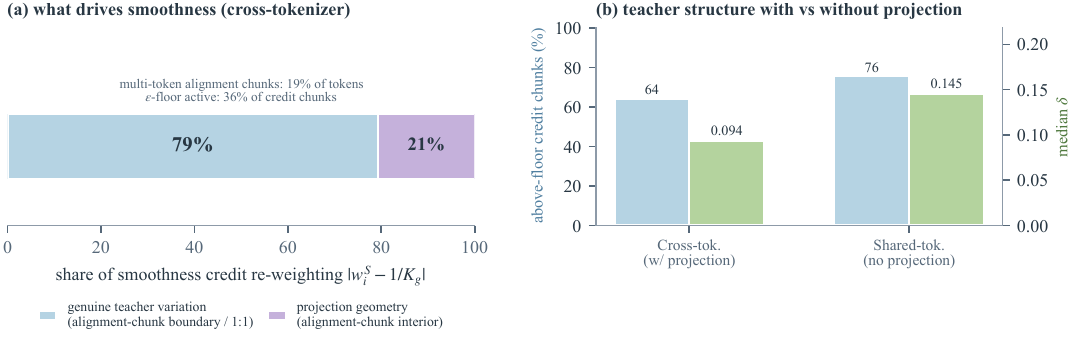}
    \caption{\textbf{Smoothness primarily reflects teacher structure rather
    than projection geometry.}
    \textbf{(a)} On cross-tokenizer rollouts (Mistral $\leftarrow$ Qwen3-8B),
    $79\%$ of the credit re-weighting induced by smoothness-only allocation,
    $|w_i^S-1/K_g|$, falls on teacher-variation positions (alignment-chunk boundaries
    and $1{:}1$ positions), while $21\%$ falls at projection-flat alignment-chunk
    interiors. Alignment chunks containing multiple student tokens cover $19\%$
    of tokens; the $\epsilon$ floor is active in $36\%$ of credit chunks.
    \textbf{(b)} Cross-tokenizer projection is compared with shared-tokenizer
    Qwen3-8B $\rightarrow$ Qwen3-4B, where $M=N$ and no projection is used.
    The shared-tokenizer setting has more credit chunks whose mean variation exceeds
    the floor and a larger median $\delta$. Both panels are computed offline
    from saved rollouts and teacher scores, without additional training.}
    \label{fig:smoothness_attribution}
\end{figure}

\paragraph{Teacher-signal variation and projection.}
Equation~\eqref{eq:teacher_projection} in Appendix~\ref{app:chunk_construction}
assigns each alignment chunk's aggregate
teacher score uniformly to its student tokens. Although a credit chunk can cross
alignment-chunk boundaries, the projected signal is constant within each
alignment chunk. The smoothness term could therefore reward alignment
geometry rather than genuine teacher structure. We test this explanation with
two analyses that require no retraining (Figure~\ref{fig:smoothness_attribution}).
First, alignment chunks containing multiple student tokens cover $19\%$ of
student tokens in the cross-tokenizer rollouts. Defining smoothness-only credit as
$w_i^S=s_i/\sum_{j\in g}s_j$, $79\%$ of its re-weighting mass
$|w_i^S-1/K_g|$ falls on teacher-variation positions (alignment-chunk boundaries
and $1{:}1$ positions), compared with $21\%$ at projection-flat alignment-chunk
interiors. The $\epsilon$ floor is active in $36\%$ of credit chunks, limiting the influence of
small local variations.

Second, the shared-tokenizer setting removes projection entirely because
$M=N$. Its teacher signal retains at least comparable local structure: within
credit chunks, median local variation is $0.145$, compared with $0.094$
cross-tokenizer, and mean variation exceeds $\epsilon_{\min}$ in $76\%$ of
credit chunks, compared with
$64\%$. The median within-chunk smoothness-weight spread is also similar
($s_{\max}/s_{\min}=2.68$ versus $2.60$). Together, these analyses indicate
that smoothness responds primarily to teacher-signal variation, while
projection geometry contributes a minority of its observed re-weighting.

\FloatBarrier

\section{Training and Evaluation Details}
\label{app:hyperparams}

Table~\ref{tab:hparams} lists the main cross-tokenizer training configuration.
Tables~\ref{tab:main}--\ref{tab:credit_components} report checkpoint-level point
estimates. Tables~\ref{tab:seed_summary} and~\ref{tab:llama_seed_summary} summarize
three evaluated seeds for four Mistral and two Llama configurations;
Appendix~\ref{app:seed_robustness} provides the
individual checkpoints, scores, and collapse records.
Training compute varies across methods, which share the evaluation protocol
described below.

\paragraph{Training protocol.}
Unless otherwise specified, training uses a prompt batch size of 64,
learning rate $10^{-6}$, and maximum prompt and response lengths of
1,024 and 16,384 tokens, respectively. Heterogeneous pairs use regroup size
$K{=}6$. By default, CAST uses multi-rollout sampling ($n{=}2$): both
independently sampled responses per prompt contribute to training, without
best-response selection or cross-response advantage normalization. The teacher
only provides log-probabilities on student responses.

\begin{table}[!htbp]
\caption{Training hyperparameters for CAST (cross-tokenizer setting).}
\label{tab:hparams}
\centering
\footnotesize
\begin{tabular}{ll}
\toprule
Hyperparameter & Value \\
\midrule
Teacher / students & Qwen3-8B / \{Mistral-7B, Llama-3.1-8B\} \\
Student initialization & SFT on OpenThoughts 400K \\
Prompt-set sampling seed (v1 / v2) & 44 \\
Optimizer learning rate & $1\times10^{-6}$ \\
LR warmup & 10 steps \\
Weight decay & 0.01 \\
Gradient clipping & 1.0 \\
PPO clip ($\varepsilon_{\mathrm{low}}$ / $\varepsilon_{\mathrm{high}}$) & 0.2 / 0.28 \\
Advantage clamp & $\pm 10.0$ \\
Rollouts per prompt $n$ & 2 \\
Prompt batch size & 64 \\
Max prompt / response length & 1024 / 16384 tokens \\
Regroup size $K$ & 6 \\
Entropy floor $\eta$ & $10^{-8}$ \\
Smoothness floor $\epsilon_{\min}$ & 0.1 \\
KL anchor coefficient $\lambda_{\mathrm{KL}}$ & 0.2 \\
Loss aggregation & token-mean \\
Rollout temperature (train / eval) & 1.0 / 0.6 \\
Total steps & 250 \\
\bottomrule
\end{tabular}
\end{table}

\paragraph{Compute and token budget.}
We implement training in the \texttt{verl} RL framework. Each run uses one
node with eight H100 GPUs: four score responses with the frozen Qwen3-8B
teacher using tensor parallelism, and four perform student rollout generation
and policy updates. Training runs for $250$ optimizer steps, with $64$ prompts
and $n$ rollouts per prompt at each step
(Table~\ref{tab:credit_components}). The maximum prompt and response lengths
are $1{,}024$ and $16{,}384$ tokens, respectively, and the learning rate is
$10^{-6}$.

For the default CAST configuration ($n{=}2$), the mean response length is
$11.1$k tokens. Each step therefore generates
$\approx64\times2\times11.1\text{k}\approx1.42$M response tokens, or
$\approx1.48$M actor tokens including prompts. Over training, these totals are
$\approx355$M generated tokens and $\approx369$M actor tokens. The teacher
scores each rollout once, giving $64\times2\times250=32{,}000$ scoring forwards
over $\approx355$M teacher tokens. Wall-clock time is $\approx30$~h per run,
or $\approx437$~s per step. The $n{=}1$ configurations (DPCA, single rollout,
and no anchor$+$single) use half this generation and teacher-scoring budget:
$\approx178$M tokens and $16{,}000$ teacher forwards. The $n{=}4$ configuration
uses twice the $n{=}2$ budget. At a given $n$, ULD, MinED, and ALM share the
generation and teacher-scoring volumes but differ in alignment and loss
computation.

\paragraph{Checkpoint selection.}
We save checkpoints every $25$ steps and evaluate them on all six benchmarks.
The point-estimate comparisons use one selected checkpoint per configuration.
For the four repeated Mistral configurations, Tables~\ref{tab:main}
and~\ref{tab:credit_components} use the highest Avg-6 across evaluated
seed--checkpoint pairs. Table~\ref{tab:seed_summary} instead averages each seed's
highest-Avg-6 evaluated checkpoint (details in Appendix~\ref{app:seed_robustness}).
For Llama CAST and ULD, Table~\ref{tab:main} uses the highest Avg-6
among the three per-seed checkpoints in Table~\ref{tab:llama_seed_records};
Table~\ref{tab:llama_seed_summary} summarizes all three.
Benchmark-based selection can make the reported averages optimistic.

\paragraph{Evaluation protocol and run variability.}
We evaluate AIME 2024, 2025, and 2026 \citep{aime26}, MATH-500
\citep{hendrycks2021measuring}, GPQA-Diamond
\citep{rein2024gpqa}, and LiveCodeBench
\citep{jain2025livecodebench} at temperature $0.6$ and top-$p$ $0.95$.
We estimate pass@1 using 32 samples per AIME problem and four per problem on
other benchmarks. Avg-6 is the unweighted mean of the six displayed
one-decimal scores, rounded to two decimals. We also inspect individual
benchmarks to distinguish aggregate gains from domain-specific effects. The annual
AIME sets contain $30$, $30$, and $15$ problems for 2024, 2025, and 2026,
respectively. Each problem accounts for $3.3$, $3.3$, and $6.7$ percentage
points in the corresponding score. MATH-500 ($500$ problems), GPQA-Diamond
($198$), and LiveCodeBench ($960$) provide considerably more stable estimates.
We therefore emphasize patterns that repeat across AIME years, the other
benchmarks, both students (Mistral and Llama), and the component ablations
(Table~\ref{tab:credit_components}), together with the mechanism analysis.
Tables~\ref{tab:seed_summary} and~\ref{tab:llama_seed_summary} report means and
sample standard deviations across three evaluated seeds for four Mistral and
two Llama configurations; Tables~\ref{tab:seed_robustness}
and~\ref{tab:llama_seed_records} list their individual scores. Other comparisons
use single runs. Figure~\ref{fig:training_dynamics} gives separate training diagnostics.

\clearpage
\section{Case Study}
\label{app:case_study}

The following cases compare the SFT initialization, DPCA, ULD, and CAST on
four selected problems from the 2024 AIME I. Each panel presents a condensed question,
excerpts from the final $240$ characters of each recorded response, the
extracted-answer outcome, and an analysis of the visible behavior.
Ellipses indicate omitted text, with mathematical formatting normalized;
-- denotes no extracted final answer. The examples illustrate repetition
and answer completion, complementing the training traces in
Figure~\ref{fig:training_dynamics} and aggregate results in Table~\ref{tab:main}.

\begingroup
\definecolor{CasePanelBlue}{HTML}{315D79}
\definecolor{CasePanelTint}{HTML}{EFF5F8}
\definecolor{CasePanelBorder}{HTML}{B8C6CE}
\tcbset{castcase/.style={colback=white,colframe=CasePanelBorder,
colbacktitle=CasePanelBlue,coltitle=white,fonttitle=\bfseries,
boxrule=0.45pt,arc=1mm,left=7pt,right=7pt,top=5pt,bottom=5pt,
before skip=7pt,after skip=7pt}}

\begin{tcolorbox}[castcase,title={Case study 1: Grid paths and repeated direction sequences}]
\label{case:aime_21}
\textbf{Question.} How many length-$16$ paths from the lower-left to the upper-right corner of an $8\times8$ grid change direction exactly four times?\par
\textbf{Reference answer:} $294$.\hfill\textit{\href{https://artofproblemsolving.com/wiki/index.php/2024_AIME_I_Problems/Problem_6}{2024 AIME~I, Problem~6}}\par
\smallskip
{\footnotesize\setlength{\tabcolsep}{3pt}
\begin{tabularx}{\linewidth}{@{}p{0.13\linewidth}>{\raggedright\arraybackslash}Xp{0.17\linewidth}@{}}
\toprule
\textbf{Method} & \textbf{Response excerpts} & \textbf{Result} \\
\midrule
\textcolor{CasePanelBlue}{\textbf{CAST}} & Total ways for this case: $21\times7=147$ \textnormal{[\ldots]} Adding both cases: $147+147=294$ & \textcolor{CasePanelBlue}{$\checkmark$ Correct}\par $294$ \\ \addlinespace[3pt]
SFT & Suppose the path starts with R. Then, the first move is R. \textnormal{[\ldots]} If it's U, that's & --\par No final answer \\ \addlinespace[3pt]
DPCA & up, right, up, right, up, right, up, right \textnormal{[\ldots]} & --\par No final answer \\ \addlinespace[3pt]
ULD & then up, then right, then up, then right, then up, then right \textnormal{[\ldots]} & --\par No final answer \\ \addlinespace[3pt]
\bottomrule
\end{tabularx}}
\begin{tcolorbox}[colback=CasePanelTint,colframe=CasePanelTint,boxrule=0pt,
arc=0.5mm,left=5pt,right=5pt,top=4pt,bottom=4pt,before skip=4pt,after skip=0pt]
\textbf{Analysis.} The DPCA and ULD response tails repeat alternating directions without returning a count. SFT continues reconsidering individual moves. CAST completes the count by adding two symmetric cases, $147+147=294$, illustrating the contrast between repeated local enumeration and answer completion.
\end{tcolorbox}
\end{tcolorbox}
\begin{tcolorbox}[castcase,title={Case study 2: Enclosing sphere and unfinished calculations}]
\label{case:aime_2}
\textbf{Question.} Rectangular boxes have surface area $54$ and volume $23$. Let $r$ be the radius of the smallest sphere that can contain each box. If $r^2=p/q$ in lowest terms, find $p+q$.\par
\textbf{Reference answer:} $721$.\hfill\textit{\href{https://artofproblemsolving.com/wiki/index.php/2024_AIME_I_Problems/Problem_15}{2024 AIME~I, Problem~15}}\par
\smallskip
{\footnotesize\setlength{\tabcolsep}{3pt}
\begin{tabularx}{\linewidth}{@{}p{0.13\linewidth}>{\raggedright\arraybackslash}Xp{0.17\linewidth}@{}}
\toprule
\textbf{Method} & \textbf{Response excerpts} & \textbf{Result} \\
\midrule
\textcolor{CasePanelBlue}{\textbf{CAST}} & We have $r^2=\frac{657}{64}$ \textnormal{[\ldots]} $p+q=657+64=\boxed{721}$ & \textcolor{CasePanelBlue}{$\checkmark$ Correct}\par $721$ \\ \addlinespace[3pt]
SFT & Thus numerator is $16{,}583+4{,}600\sqrt{6}$, Therefore & --\par No final answer \\ \addlinespace[3pt]
DPCA & Therefore, $f(x)=2x^2+\frac{529}{27x^2-46x}$. \textnormal{[\ldots]} But & --\par No final answer \\ \addlinespace[3pt]
ULD & Multiply all terms by 4: $16-31\lambda-132.25+46\lambda=0$ \textnormal{[\ldots]} & --\par No final answer \\ \addlinespace[3pt]
\bottomrule
\end{tabularx}}
\begin{tcolorbox}[colback=CasePanelTint,colframe=CasePanelTint,boxrule=0pt,
arc=0.5mm,left=5pt,right=5pt,top=4pt,bottom=4pt,before skip=4pt,after skip=0pt]
\textbf{Analysis.} The baseline tails remain in intermediate algebra: SFT expands a radical expression, DPCA rewrites a function of $x$, and ULD manipulates an equation in $\lambda$. CAST reaches $r^2=657/64$ and returns $p+q=721$. This case contrasts ongoing intermediate calculations with completion of the requested answer.
\end{tcolorbox}
\end{tcolorbox}
\clearpage
\begin{tcolorbox}[castcase,title={Case study 3: Lottery probability and shared success}]
\label{case:aime_29}
\textbf{Question.} Jen picks four distinct numbers from $\{1,\ldots,10\}$, and four numbers are drawn uniformly. Conditional on matching at least two, the probability of matching all four is $m/n$ in lowest terms. Find $m+n$.\par
\textbf{Reference answer:} $116$.\hfill\textit{\href{https://artofproblemsolving.com/wiki/index.php/2024_AIME_I_Problems/Problem_4}{2024 AIME~I, Problem~4}}\par
\smallskip
{\footnotesize\setlength{\tabcolsep}{3pt}
\begin{tabularx}{\linewidth}{@{}p{0.13\linewidth}>{\raggedright\arraybackslash}Xp{0.17\linewidth}@{}}
\toprule
\textbf{Method} & \textbf{Response excerpts} & \textbf{Result} \\
\midrule
\textcolor{CasePanelBlue}{\textbf{CAST}} & The conditional probability is $\frac{1}{115}$ \textnormal{[\ldots]} $m+n=1+115=\boxed{116}$ & \textcolor{CasePanelBlue}{$\checkmark$ Correct}\par $116$ \\ \addlinespace[3pt]
SFT & So that's correct. Alternatively, maybe the problem is that the lottery is selecting & --\par No final answer \\ \addlinespace[3pt]
DPCA & But the problem is about the probability that her numbers are a subset of the lottery & --\par No final answer \\ \addlinespace[3pt]
ULD & The fraction $\frac{1}{115}$ is already in its simplest form \textnormal{[\ldots]} $m+n=1+115=116$ & \textcolor{CasePanelBlue}{$\checkmark$ Correct}\par $116$ \\ \addlinespace[3pt]
\bottomrule
\end{tabularx}}
\begin{tcolorbox}[colback=CasePanelTint,colframe=CasePanelTint,boxrule=0pt,
arc=0.5mm,left=5pt,right=5pt,top=4pt,bottom=4pt,before skip=4pt,after skip=0pt]
\textbf{Analysis.} ULD and CAST both finish with $1/115$ and the correct sum $116$. The SFT and DPCA tails return to interpreting the lottery setup without an extracted answer. This case shows a success shared with ULD, alongside the unfinished responses of SFT and DPCA.
\end{tcolorbox}
\end{tcolorbox}
\begin{tcolorbox}[castcase,title={Case study 4: Logarithmic system and repeated reconsideration}]
\label{case:aime_17}
\textbf{Question.} For real $x,y>1$ satisfying $\log_x(y^x)=\log_y(x^{4y})=10$, find $xy$.\par
\textbf{Reference answer:} $25$.\hfill\textit{\href{https://artofproblemsolving.com/wiki/index.php/2024_AIME_I_Problems/Problem_2}{2024 AIME~I, Problem~2}}\par
\smallskip
{\footnotesize\setlength{\tabcolsep}{3pt}
\begin{tabularx}{\linewidth}{@{}p{0.13\linewidth}>{\raggedright\arraybackslash}Xp{0.17\linewidth}@{}}
\toprule
\textbf{Method} & \textbf{Response excerpts} & \textbf{Result} \\
\midrule
\textcolor{CasePanelBlue}{\textbf{CAST}} & since the problem only asks for the value of $xy$ \textnormal{[\ldots]} $\boxed{25}$ & \textcolor{CasePanelBlue}{$\checkmark$ Correct}\par $25$ \\ \addlinespace[3pt]
SFT & Let me try t = 0. Let me see, but that's not allowed. Hmm, maybe t = 0. Let me try t = 0. \textnormal{[\ldots]} & --\par No final answer \\ \addlinespace[3pt]
DPCA & Try $k=0.5$: \textnormal{[\ldots]} $\ln(10)\approx2.3026$; $\ln(0.5)\approx-0.6931$ \textnormal{[\ldots]} & --\par No final answer \\ \addlinespace[3pt]
ULD & the product $xy=25$ is derived directly from the relationships between $x$ and $y$. \textnormal{[\ldots]} $\boxed{25}$ & \textcolor{CasePanelBlue}{$\checkmark$ Correct}\par $25$ \\ \addlinespace[3pt]
\bottomrule
\end{tabularx}}
\begin{tcolorbox}[colback=CasePanelTint,colframe=CasePanelTint,boxrule=0pt,
arc=0.5mm,left=5pt,right=5pt,top=4pt,bottom=4pt,before skip=4pt,after skip=0pt]
\textbf{Analysis.} SFT repeatedly proposes $t=0$ while acknowledging that it is not allowed, illustrating a loop of self-correction that does not resolve the problem. DPCA continues numerical substitutions without returning $xy$. ULD and CAST both conclude with $25$, completing the requested calculation.
\end{tcolorbox}
\end{tcolorbox}
\endgroup

\clearpage
\section{Performance Across Training Seeds}
\label{app:seed_robustness}

Table~\ref{tab:seed_robustness} provides the per-seed checkpoints and scores
underlying the main-text summary in Table~\ref{tab:seed_summary}.
The four Qwen3-8B $\rightarrow$ Mistral-7B configurations are DPCA, uniform
credit, uncertainty-only credit, and full CAST (Table~\ref{tab:credit_components}).
All use the same SFT initialization and training data, with
learning rate $10^{-6}$. CAST and both credit controls use $n{=}2$ and
selective anchoring ($\lambda_{\mathrm{KL}}=0.2$); original DPCA uses $n{=}1$
without anchoring. For each seed, we select the evaluated checkpoint with the
highest Avg-6, using the unrounded mean of the six displayed scores.
This yields three checkpoints per configuration, one per seed.
Table~\ref{tab:seed_summary} reports their mean and sample standard deviation;
Tables~\ref{tab:main} and~\ref{tab:credit_components} report the checkpoint with
the highest Avg-6 among these three.
These performance summaries exclude unevaluated early-terminated runs;
Table~\ref{tab:collapse} records training outcomes for all six listed seeds.

\begin{table}[!htbp]
\caption{Performance across three evaluated training seeds per configuration
for Qwen3-8B $\rightarrow$ Mistral-7B at learning rate $10^{-6}$.
Per-seed pass@1 (\%); Step identifies each seed's highest-Avg-6
evaluated checkpoint.
Aggregate means and sample standard deviations appear in
Table~\ref{tab:seed_summary}.}
\label{tab:seed_robustness}
\centering
\footnotesize
\setlength{\tabcolsep}{2pt}
\renewcommand{\arraystretch}{1.2}
\begin{tabular*}{\linewidth}{@{\extracolsep{\fill}}lrrrrrrrrr@{}}
\toprule
Method & Seed & Step & A24 & A25 & A26 & M500 & GPQA & LCB & Avg. \\
\midrule
\multirow{3}{*}{DPCA}
 & 44 & 150 & 38.3 & 37.1 & 44.4 & 73.7 & 35.7 & 16.0 & 40.87 \\
 & 653512071 & 175 & 41.8 & 35.8 & 47.7 & 73.8 & 35.4 & 13.9 & 41.40 \\
 & 893671858 & 175 & 41.0 & 34.6 & 45.0 & 73.8 & 34.1 & 9.8 & 39.72 \\
\midrule
\multirow{3}{*}{Uniform credit}
 & 44 & 150 & 45.0 & 35.2 & 47.1 & 78.0 & 36.5 & 26.5 & 44.72 \\
 & 653512071 & 175 & 42.3 & 37.2 & 46.7 & 78.9 & 38.5 & 26.8 & 45.07 \\
 & 893671858 & 200 & 45.8 & 35.9 & 44.8 & 76.6 & 36.5 & 28.3 & 44.65 \\
\midrule
\multirow{3}{*}{Uncertainty only}
 & 44 & 200 & 48.3 & 41.6 & 51.2 & 75.3 & 40.2 & 17.0 & 45.60 \\
 & 653512071 & 200 & 47.9 & 39.6 & 49.2 & 73.4 & 38.8 & 17.9 & 44.47 \\
 & 893671858 & 175 & 50.7 & 40.9 & 48.8 & 75.4 & 39.0 & 18.5 & 45.55 \\
\midrule
\multirow{3}{*}{\textbf{CAST}}
 & 44 & 200 & 45.7 & 40.9 & 51.0 & 78.5 & 39.5 & 27.7 & 47.22 \\
 & 653512071 & 175 & 46.2 & 39.5 & 50.6 & 78.1 & 40.2 & 28.1 & 47.12 \\
 & 893671858 & 150 & 44.8 & 39.3 & 48.8 & 80.4 & 38.8 & 28.0 & 46.68 \\
\bottomrule
\end{tabular*}
\end{table}

CAST's selected Avg-6 ranges from $46.68$ to $47.22$, above the highest
selected control score ($45.60$). Its mean Avg-6 exceeds uniform credit,
uncertainty-only credit, and DPCA by $2.19$, $1.80$, and $6.34$ points,
respectively, with differences computed from unrounded seed means.
The per-benchmark summaries also retain the code-performance trade-off:
CAST reaches $27.9\pm0.2$ on LCB, compared with $17.8\pm0.8$ for
uncertainty-only credit, while the latter has the higher AIME-2024 mean.

\paragraph{Llama CAST and ULD.}
Table~\ref{tab:llama_seed_records} gives the recorded checkpoints for seeds
$44$, $653512071$, and $893671858$ with Qwen3-8B as teacher.
Table~\ref{tab:llama_seed_summary} reports the mean and sample standard deviation
across these three runs, computing Avg-6 statistics from each seed's unrounded
mean of the six displayed benchmark scores.
The highest Avg-6 is $43.87$ for CAST (seed $653512071$, step $250$)
and $42.55$ for ULD (the same seed, step $225$); these checkpoints supply
the corresponding complete rows in Table~\ref{tab:main}.

\begin{table}[!htbp]
\caption{Per-seed Llama-3.1-8B pass@1 (\%) with Qwen3-8B as teacher.
Step identifies the reported checkpoint; summary statistics appear in
Table~\ref{tab:llama_seed_summary}.}
\label{tab:llama_seed_records}
\centering
\footnotesize
\setlength{\tabcolsep}{2pt}
\begin{tabular*}{\linewidth}{@{\extracolsep{\fill}}lrrrrrrrrr@{}}
\toprule
Method & Seed & Step & A24 & A25 & A26 & M500 & GPQA & LCB & Avg. \\
\midrule
\multirow{3}{*}{ULD}
 & 44 & 200 & 41.6 & 34.5 & 41.9 & 75.1 & 33.1 & 25.6 & 41.97 \\
 & 653512071 & 225 & 41.6 & 36.2 & 41.9 & 75.1 & 34.7 & 25.8 & 42.55 \\
 & 893671858 & 175 & 41.0 & 34.1 & 40.2 & 73.7 & 35.7 & 25.6 & 41.72 \\
\midrule
\multirow{3}{*}{\textbf{CAST}}
 & 44 & 250 & 42.7 & 37.2 & 43.8 & 74.1 & 36.9 & 26.6 & 43.55 \\
 & 653512071 & 250 & 44.1 & 37.7 & 45.0 & 75.0 & 34.5 & 26.9 & 43.87 \\
 & 893671858 & 200 & 43.1 & 39.0 & 41.7 & 75.2 & 36.7 & 26.7 & 43.73 \\
\bottomrule
\end{tabular*}
\end{table}

\clearpage
\begin{table}[H]
\caption{Training collapse across six seeds (Qwen3-8B $\rightarrow$
Mistral-7B, lr $10^{-6}$). $\checkmark$: completed $250$ steps with healthy
scores; $\times$: repetition collapse (RC) or entropy collapse (EC).
Rates count collapsed runs out of six.}
\label{tab:collapse}
\centering
\footnotesize
\setlength{\tabcolsep}{4pt}
\begin{tabular}{lccccccc}
\toprule
\multirow{2}{*}{Method} & \multicolumn{6}{c}{Seed} & \multirow{2}{*}{Collapse rate} \\
\cmidrule(lr){2-7}
 & 44 & 45 & 46 & 653512071 & 893671858 & 417698746 & \\
\midrule
DPCA (vanilla) & $\checkmark$ & $\checkmark$ & $\checkmark$ & $\checkmark$ & $\checkmark$ & $\checkmark$ & $0/6$ ($0.0\%$) \\
Uniform credit & $\checkmark$ & $\times$ (RC) & $\times$ (EC) & $\checkmark$ & $\checkmark$ & $\times$ (EC) & $3/6$ ($50.0\%$) \\
Uncertainty only & $\checkmark$ & $\times$ (RC) & $\times$ (EC) & $\checkmark$ & $\checkmark$ & $\checkmark$ & $2/6$ ($33.3\%$) \\
\textbf{CAST} & $\checkmark$ & $\checkmark$ & $\times$ (RC) & $\checkmark$ & $\checkmark$ & $\checkmark$ & $1/6$ ($16.7\%$) \\
\bottomrule
\end{tabular}
\end{table}

\begin{figure}[H]
\centering
\includegraphics[width=\linewidth]{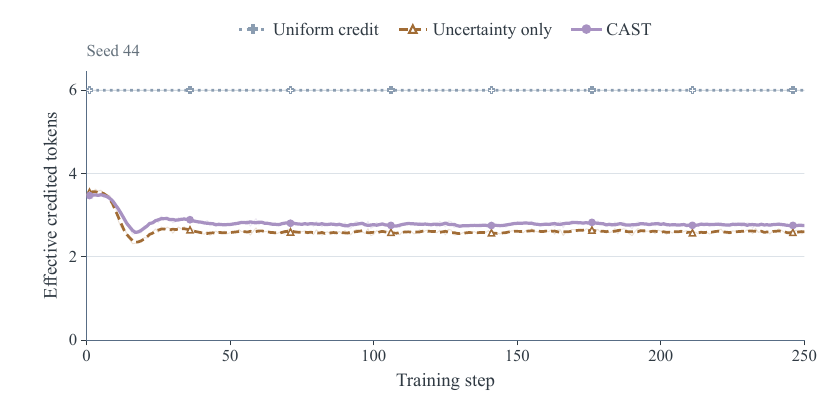}
\caption{\textbf{CAST has slightly broader effective credit support than
uncertainty-only weighting in seed-$44$ diagnostic runs.}
Using the seed-$44$ diagnostic trajectories from Figure~\ref{fig:training_dynamics},
we compare uniform credit, uncertainty-only weighting, and CAST.
For each chunk, the effective token count is the inverse participation ratio
$N_{\mathrm{eff}}(g)=1/\sum_{i\in g}w_i^2$, with normalized weights
$w_i=\widehat A_i/A_g$.
It equals $K_g$ under uniform allocation and approaches $1$ as credit
concentrates on one token, rather than counting tokens above a threshold.
The monitor averages this quantity over chunks with $K_g>1$ and
$|A_g|>10^{-8}$; faint traces show the exported training-step values, and
the overlaid curves show five-step moving averages.}
\label{fig:credit_concentration}
\end{figure}

\paragraph{Credit concentration and training stability.}
Over steps $1$--$250$ of the seed-$44$ diagnostic runs, CAST averages $2.814$
effective tokens versus $2.632$ for uncertainty-only weighting, a $6.9\%$
increase, while uniform credit averages $5.999$
(Figure~\ref{fig:credit_concentration}). Alongside the lower observed collapse
count across six seeds ($1/6$ versus $2/6$; Table~\ref{tab:collapse}), this
is consistent with smoothness moderating uncertainty-driven credit concentration.
The bound in Eq.~\eqref{eq:smoothness_bound} controls the relative reweighting
introduced by smoothness; it does not guarantee diffuse product weights.
Uniform credit remains more diffuse yet collapses in $3/6$ runs, indicating
that credit dispersion alone does not explain training stability.

\end{document}